\documentclass{article}
\def\bs{\ensuremath\boldsymbol}

\usepackage{PRIMEarxiv}

\usepackage[utf8]{inputenc} 
\usepackage[T1]{fontenc}    
\usepackage{hyperref}       
\usepackage{url}            
\usepackage{booktabs}       
\usepackage{amsfonts}       
\usepackage{nicefrac}       
\usepackage{centernot}
\usepackage{microtype}      
\usepackage{lipsum}
\usepackage{amsrefs}
\usepackage[ruled,vlined]{algorithm2e}
\usepackage{amsmath, amsthm, amssymb, amsfonts,ulem, multicol,cancel,mathtools,esint}
\usepackage{cleveref}
\usepackage{fancyhdr}       
\usepackage{graphicx}       
\graphicspath{{media/}}     

\newtheorem{theorem}{Theorem}[section]

\newtheorem{lemma}{Lemma}[section]

\newtheorem{example}{Example}[section]

\newtheorem{assumption}{Assumption}[section]

\newtheorem{defn}{Definition}[section]
\newtheorem{rem}{Remark}[section]
\newtheorem{prop}{Proposition}[section]
\newtheorem{cor}{Corollary}[section]

\title{A Local Convergence Theory for the Stochastic Gradient Descent  Method in Non-Convex Optimization with non-isolated local minima
\thanks{\textit{\underline{Citation}}: 
\textbf{Authors. Title. Pages.... DOI:000000/11111.}} 
}

\author{
  Taehee Ko,\\
  Department of Mathematics \\
  Pennsylvania State University\\
  State College\\
  \texttt{tuk351@psu.edu} \\
   \And
  Xiantao Li \\
  Department of Mathematics \\
  Pennsylvania State University\\
  State College\\
  \texttt{xxl12@psu.edu} \\
}

\begin{document}
\maketitle

\begin{abstract}
Loss functions with non-isolated minima have emerged in several machine
learning problems, creating a gap between theory and practice. In this paper, we formulate a new type of local convexity condition that is suitable to describe the behavior of loss functions near non-isolated minima. We show that such condition is general enough to encompass many existing conditions. In addition we study the local convergence of the SGD under this mild condition by adopting the notion of stochastic stability. The corresponding concentration inequalities from the convergence analysis help to interpret the empirical observation from some practical training results. 
\end{abstract}


\section{Introduction}
The stochastic gradient descent (SGD) and its variants have been predominantly applied in machine learning \cite{bottou2018optimization} due to the overall computational efficiency and robustness. In a typical training task, the objective function $f(\bs{x})$ is expressed as an empirical risk,
\begin{equation}\label{eq: fi2f}
    f(\bs{x}) = \mathbb{E}[f(\bs{x};\omega)].
\end{equation}Here, the randomness of $\omega$ stems from the sampling of training data set.
The standard SGD updates the iterations as follows:
\begin{equation}\label{eq: sgd}
        \textrm{SGD}:
        \bs{x}_{n+1}=\bs{x}_n-a_n \left(\nabla f(\bs{x}_n)+\bs{\xi}_n\right), 
\end{equation}where we have adopted the notation from \cite{mertikopoulos2020almost}, and for simplicity we set the batch size to be 1. Here, $a_n$ is known as the learning rate, and $\bs{\xi}_n:=\nabla f(\bs{x}_n,\omega_n)-\nabla f(\bs{x}_n)$ is the noise induced by the sampling of the gradient.

In the case when $f$ is convex, convergence properties for SGD have been well established, examples include the stepsize policy \cites{bottou2018optimization}, comparison to stochastic averaging methods \cite{nemirovski2009robust}, validation analysis \cite{ghadimi2012optimal}, etc. On the other hand, a remarkable observation, as demonstrated in many recent studies, is that the stochastic gradient algorithms still perform well in non-convex optimization, e.g. the training of neural networks \cites{neyshabur2015path,goyal2017accurate,cutkosky2019momentum}, and
such success has intrigued many theoretical development to explore convergence properties for non-convex optimization problems. It is important though, to point out that most of these results are obtained under some global assumptions, such as the global Lipschitz constant \cites{li2019convergence,haddadpour2019local,madden2020high} or global H\"{o}lder constant \cite{lei2019stochastic} for the gradient or the globally bounded variance of the noise \cite{karimi2016linear}.

Despite these important theoretical results, training tasks with actual non-convex loss functions have exhibited many issues that are still difficult to interpret based on these analyses. For instance, the performance of training neural networks can be very sensitive to the initialization \cites{glorot2010understanding,li2018visualizing,fort2019goldilocks}. Intuitively, upon convergence, the behavior of SGD is largely determined by local properties. However, many of the existing results on non-convex optimizations are established under global assumptions, which may not hold for many practical training tasks, e.g., the lack of a global Lipschitz constant \cites{sun2019optimization,hodgkinson2021generalization}, and globally bounded variance of the noise \cite{gurbuzbalaban2021heavy}.
One recent development toward addressing this issue is the work \cite{mertikopoulos2020almost} who proved convergence of SGD under local assumptions on the initialization, the gradient and the noise, but in the case of isolated local minima.

In addition to the aforementioned issues, another interesting issue in non-convex optimization is the presence of non-isolated minima, as observed in  \cites{sagun2017empirical,cooper2021global,draxler2018essentially,garipov2018loss}. Despite the clear intuition from Figure \ref{fig:my_label1},  formulating local convexity conditions around such a set to precisely describe the behavior of the loss function  in \eqref{eq: fi2f} is a non-trivial task. More importantly, it creates an important gap between the practical optimization and theoretical analysis.
This important scenario has received little attention until the recent analysis under some local convexity assumptions  \cites{patel2021global,fehrman2020convergence}. Meanwhile, questions still remain as to whether more general characterizations exist for a wider variety of loss functions with non-isolated local minima, and more importantly, whether the convergence SGD can still be established for these general optimization problems. From a practical viewpoint, such convergence analysis  will help interpret many observations from training algorithms, e.g., the roles of hyperparameters in the stochastic optimization.

\begin{figure}
    \centering
    \includegraphics[scale=0.25]{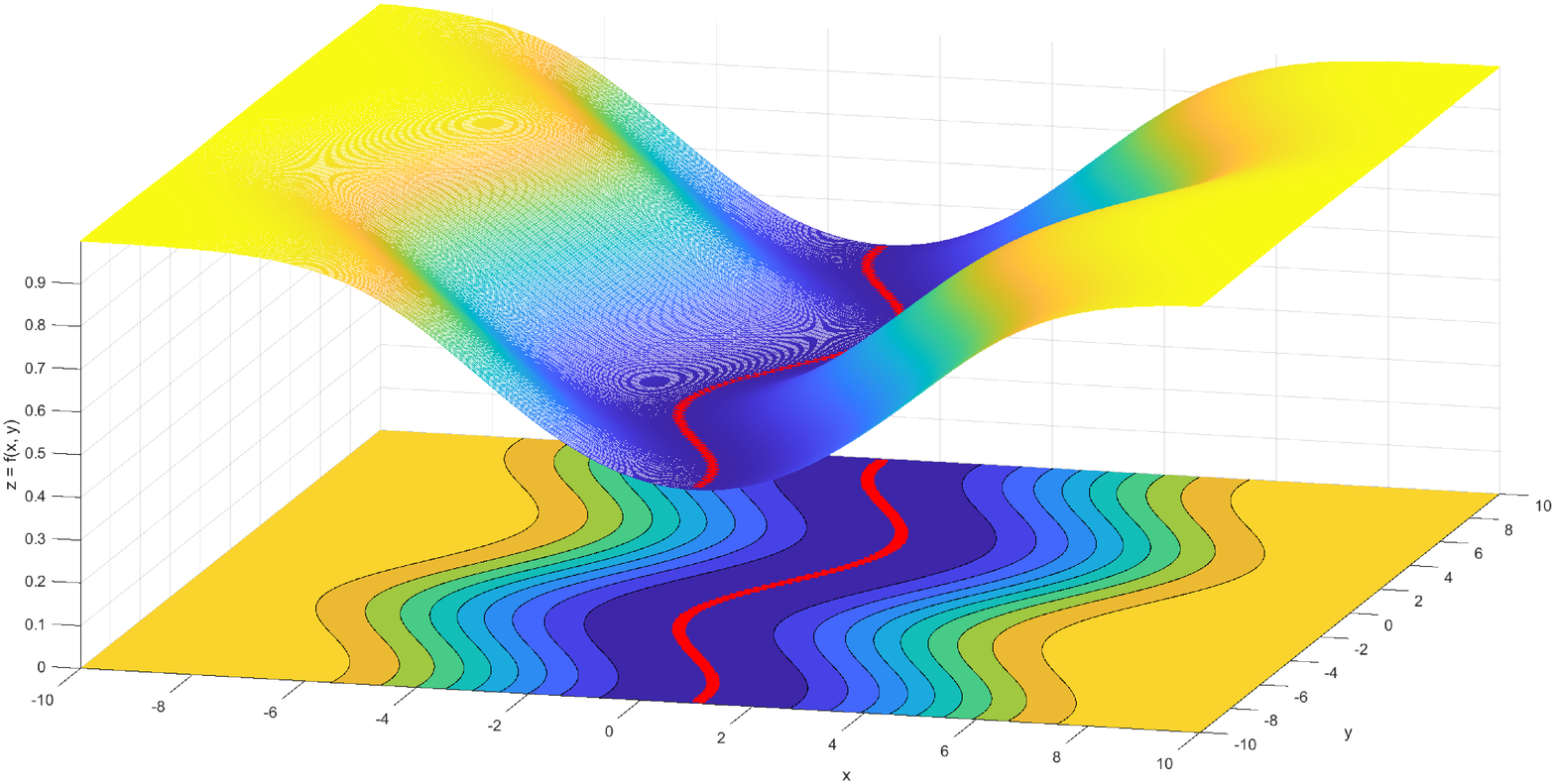}
    \caption{The landscape of a non-convex function with non-isolated global minima (red)}
    \label{fig:my_label1}
\end{figure}

The main objective of this paper is two-fold: First, we introduce locally-defined mathematical conditions for the loss function with non-isolated minima. We establish relationships among such conditions, including those proposed in prior works \cites{fehrman2020convergence,hinder2020near,karimi2016linear}. Second, we prove that non-isolated minima still has finite probability of attracting SGD iterates \eqref{eq: sgd}. Specifically, we  establish concentration properties of SGD iterations, which lead to convergence to $\mathcal{X}$ in probability and complexity bounds.  These results are obtained only with local conditions, which sharpens results from those analysis obtained under global conditions. In the case where $\mathcal{X}$ is compact, under more relaxed assumptions than those in \cite{fehrman2020convergence}, we prove a more general concentration inequality of SGD. Furthermore, even under the same assumption as in \cite{fehrman2020convergence}, our concentration inequalities are improved.

\begin{table}
{\footnotesize
\begin{tabular}{|p{4.4cm} | p{4.5cm} | p{5cm} |} 
 \hline
 Condition & Definition  & Concentration Inequality\\ [0.5ex]
 \hline
 LSC {(\cite{mertikopoulos2020almost})} & 
 $(\nabla f(\bs{x}),\bs{x}-\bs{x}^*)\geq \mu \|\bs{x}-\bs{x}^*\|^2$ &
 
 $\mathbb{P}\{\|\bs{x}-\bs{x}^*\|^2\geq\epsilon\}\leq\mathcal{O}(\frac{1}{\epsilon n^{\beta}})$ \\[8pt]
 \hline
 
 PL*  (\cite{liu2022loss}) 
 & $\|\nabla f(\bs{x})\|^2\geq \mu (f(\bs{x})-f^*)$ & $\mathbb{P}\{f(\bs{x})-f^*\geq\epsilon\}\leq\mathcal{O}(\frac{\rho^n}{\epsilon})$ \\[8pt]
 \hline
 HCPRC (\cite{fehrman2020convergence}) & $\textrm{rank}(\nabla^2f(\bs{x}))=d-\mathfrak{d}$ 
 & $\mathbb{P}\{f(\bs{x})-f^*\geq\epsilon\}\leq\mathcal{O}(\frac{1}{\epsilon^2n^{\beta}}+n^{1-\beta})$ \\[8pt]
 \hline
 LRSI  (this work) & $(\nabla f(\bs{x}),\bs{x}-\bs{x}_p)\geq \mu \|\bs{x}-\bs{x}_p\|^2$ & $\mathbb{P}\{f(\bs{x})-f^*\geq\epsilon\}\leq\mathcal{O}(\frac{1}{\epsilon n^{\beta}})$ \\[8pt]
 \hline
  WQC  (this work) & $(\nabla f(\bs{x}),\bs{x}-\bs{x}_p)\geq \zeta (f(\bs{x})-f^*)$ & $\mathbb{P}\{f(\bs{x})-f^*\geq\epsilon\}\leq\mathcal{O}(\frac{1}{\epsilon n^{1-\beta}})$ \\[8pt]
 \hline
 NNS (this work) & $(\nabla f(\bs{x}),\bs{x}-\bs{x}_p)\geq h(\bs{x})\geq 0$ & $\mathbb{P}\{h(\bs{x})\geq\epsilon\}\leq\mathcal{O}(\frac{1}{\epsilon n^{1-\beta}})$ \\[8pt]
 \hline
\end{tabular}
 }
 \caption{The concentration results by previous works and this work under \emph{local} conditions: the Local Strong Convexity (LSC), a locally-defined Polyak-{\L}ojasiewicz condition (PL*), the Hession of the constant positive rank (HCPRC) on a Compact submanifold of global minima, a locally-defined restrict secant inequality (LRSI), the weak quasar-convexity (WQC) as an extension of the quasar convexity (QC) in \cite{hinder2020near} and the non-negative support (NNS).  In these conditions, $\bs{x}^*$ refers to a local minimum, and $\bs{x}_p$ denotes a metric projection of $\bs{x}$.}
 \end{table} 
 
Specifically, we summarize our contributions as follows:

\begin{itemize}
  \item We propose several conditions that describe loss functions with non-isolated minima, including a locally-defined restrict secant inequality (LRSI), the weak quasar-convexity (WQC) the non-negative support (NNS) conditions. Their mathematical expressions, as well as those from the existing literature are summarized in Table 1.
  
  \item We propose a new approach to analyze SGD iterations. The approach combines the Lyapunov function approach \cite{kushner2003stochastic} based on the optional stopping theorem \cite{williams1991probability}, and yields nearly sharp theoretical results under only locally defined conditions. With the help of these  results, we explicitly formulate the probability of convergence when the iteration from SGD starts near a set of non-isolated  minima. This provides theoretical supports for several empirical observations from the training of neural networks.   
  
    \item We show the concentration properties of SGD that are also applicable to the regime of small batch-sizes for a large class of landscapes arisen in non-convex optimization. As an example, the result qualitatively elucidates the slowdown of SGD in a highly flat landscape near global minima. Additionally, we found that convergence properties of SGD near global minima  are still very similar to those from strongly convex functions when the local landscape of non-isolated global minima satisfies the LRSI condition.  
\end{itemize}

\section{Preliminaries}\label{prelim}

\subsection{Notation}

Throughout the paper,
$K \subset \mathbb{R}^d$ denotes a compact set unless it is stated otherwise. $d$ represents the dimension of the parameter space.  We use $\|\cdot\|$ for the $\ell_2$ norm in $\mathbb{R}^d$. $[N]$ is the set of integers $[N]:=\{1,2,\cdots, N\}.$ For the complexity estimates, we define big-$\mathcal{O}$ notation as follows: $a_n=\mathcal{O}(b_n)$ if $\underset{n}{\limsup}\left|\frac{a_n}{b_n}\right|<\infty$.

Next, we introduce the notions of the stable path and the stochastic stability for discrete stochastic processes.
\begin{defn}[Stable Path]\label{stabpath}
With an initial iterate $\bs{x}_1\in K$, a realization $\{\bs{x}_n\}_{n=2}^{\infty}$ from a stochastic algorithm is called  a stable path if
\begin{equation}
    \bs{x}_n\in K \textrm{ for all }n\geq 2.
\end{equation}
\end{defn}
Since the update rule for SGD in \eqref{eq: sgd} uses the previous information, the iteration forms a discrete stochastic process with a filtration $\{\mathcal{F}_n\}_{n=1}^{\infty}$. Formally, the stochastic stability is defined by the measure of the set of stable paths with respect to the $\sigma$-algebra $\mathcal{F}_{\infty}:=\sigma(\cup_n\mathcal{F}_n)$ as follows.
\begin{defn}[Stochastic Stability]\label{stocstab}
With an initial guess $\bs{x}_1\in K$, an iteration from a stochastic algorithm is said to be stable with probability at least $1-\eta$, if the following inequality is satisfied
\begin{equation}
    \mathbb{P}\left\{\bs{x}_n\in K\textrm{ for all }n\geq 2\;|\bs{x}_1\right\}\geq 1-\eta.
\end{equation}
\end{defn}
The notion of stochastic stability can be found in [\cite{kushner1967stochastic}, p 31]. 

\subsection{Assumptions}
Our regularity assumption on the loss function $f$ is that it has a locally Lipschitz gradient. For example, a large class of neural networks satisfy this condition. 
\begin{assumption}\label{as:01}
The gradient of $f$ in the iteration \eqref{eq: sgd} is locally Lipschitz continuous for any compact set $K$, i.e., there exists a constant $L_{K}>0$ such that,
\[\left\|\nabla f(\bs{x})-\nabla f(\bs{y})\right\|\leq L_{K}\|\bs{x}-\bs{y}\|, \quad \textrm{for all }\bs{x},\bs{y}\in K.\]
\end{assumption}

We make assumption on the noise in SGD \eqref{eq: sgd} as follows. 
\begin{assumption}\label{as:02}
The noise satisfies that
\begin{itemize}
    \item[(i)]$\mathbb{E}\left[\bs{\xi}(\bs{x},\omega)|\bs{x}\right]=0$ for any $\bs{x}\in\mathbb{R}^d$, (Unbiased Stochastic Gradient)
    \item[(ii)]There exists some $\sigma_{K}>0$ for any compact set $K$ such that \[\sup_{\bs{x}\in K}\mathbb{E}\left[\|\bs{\xi}(\bs{x},\omega)\|^2|\bs{x}\right]\leq \sigma_{K}.\]
\end{itemize}
\end{assumption}
The second part implies that the variance of the noise is locally bounded. This assumption is equivalent to the local conditions in \cite{fehrman2020convergence}). 

Next, we make the standard assumption on the learning rates (\cites{bottou2018optimization,robbins1951stochastic}),  as follows. 
\begin{assumption}\label{as:03}
The learning rates satisfy that
\begin{equation}
    a_n> 0,\quad \sum_{n=1}^{\infty}a_n=\infty,\quad \sum_{n=1}^{\infty}a_n^2<\infty.
\end{equation}
\end{assumption}

\section{Stochastic Gradient Descent with Local Convexity Conditions}\label{sec: sgd}

In this section, we start by introducing a condition that is mild, but still sufficient to describe non-convex loss function near local minima. We present the relationships between the new condition and other existing characterizations of such loss functions. Then, under such conditions, we prove the concentration inequalities of the SGD method \eqref{eq: sgd}.

\subsection{Local Convexity Conditions and their relations}
Before we present the analysis of the SGD iterations \eqref{eq: sgd}, we first examine a number of important local convexity conditions. An important emphasis will be placed on their relations, since results that are established under one condition can be directly extended to situations where a weaker condition holds.

We first introduce the notion of a metric projection, which will be used to pose a new local convexity condition.



\begin{defn}\label{defn:proj}
For a compact set $\mathcal{X}$, the set-valued mapping $\Pi_{\mathcal{X}}:\mathbb{R}^d\to 2^{\mathcal{X}}$ is defined by
\begin{equation*}
    \Pi_{\mathcal{X}}(\bs{x})=\{\bs{z}\in \mathcal{X}:\|\bs{x}-\bs{z}\|=\mathrm{dist}(\bs{x},\mathcal{X})\}.
\end{equation*}This is known as the metric projection \cite{shapiro1994existence} with respect to the usual Euclidean distance in our case. 
\end{defn}

\begin{figure}
    \centering
    \includegraphics[scale=0.32]{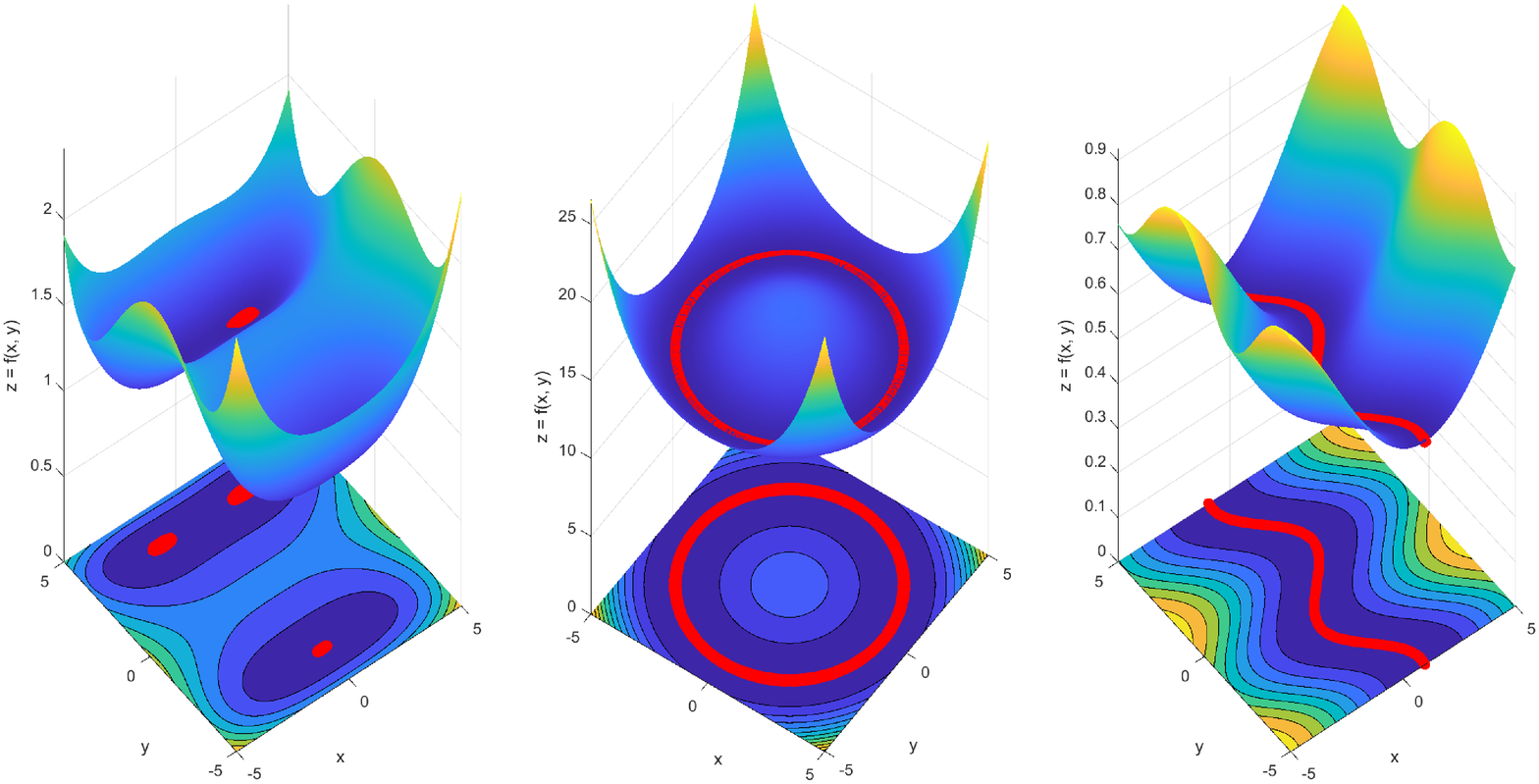}
    \caption{Examples of non-convex functions with isolated minima (left), or non-isolated minima (middle and right). The red regions indicate sets of local minima. The functions from all three figures satisfy the NNS condition \eqref{gc}.}
    \label{fig:my_label2}
\end{figure}

We will denote by $\mathcal{X}$, a compact subset of global minima with or without connectivity.  In the following assumption, we denote by $N_r(\mathcal{X})$ the closure of a $r$-neighborhood of $\mathcal{X}$.

\begin{assumption}\label{as:04}
There exists a non-negative function $h:N_r(\mathcal{X})\to \mathbb{R}$ such that
\begin{equation}\label{gc}
    \left(\nabla f(\bs{x}),\bs{x}-\bs{x}_p \right)\geq h(\bs{x}),
\end{equation}for any $\bs{x}\in N_r(\mathcal{X})$ and any projection $\bs{x}_p\in \Pi_{\mathcal{X}}(\bs{x})$. Due to the non-negative function $h$, we refer to this condition as  non-negative support $\mathrm{(NNS)}$ condition.
\end{assumption}

As a pictorial illustration, \cref{fig:my_label2} shows the examples that fulfill Assumption \ref{as:04}.

As a starting point to analyze optimization algorithms for non-convex functions, numerous conditions have been previously proposed to describe a non-convex loss function near local minima:
[\cite{zhang2013gradient}, the restricted secant inequality (RSI)], [\cite{karimi2016linear}, the Polyak-{\L}ojasiewicz (PL), the Quadratic Growth (QG)], [\cite{hinder2020near}, the quasar convexity (QC)] and [\cite{nesterov2006cubic}, the star convexity ($^*$C)]. Some of these conditions were initially proposed to hold globally, but they can be easily relaxed to {\it local} conditions, in that they are assumed to hold in the neighborhood $N_r(\mathcal{X})$. Namely, for any $\bs{x}\in N_r(\mathcal{X})$,
\begin{subequations}\label{conditions}
\begin{eqnarray}
    &\big(\nabla f(\bs{x}),\bs{x}-\bs{x}_p\big)\geq\mu\big\|\bs{x}-\bs{x}_p\big\|^2\; (\mathrm{LRSI}),\label{cond-a}\\
    &\big\|\nabla f(\bs{x})\big\|^2\geq \mu \big(f(\bs{x})-f^*)\; (\mathrm{PL}^*\big),\label{cond-b}\\
    &f(\bs{x})-f^*\geq \mu\big\|\bs{x}-\bs{x}_p\big\|^2\; (\mathrm{QG}^*),\label{cond-c}\\
    &\big(\nabla f(\bs{x}),\bs{x}-\bs{x}^*\big)\geq f(\bs{x})-f^*\; (^*\mathrm{C}),\label{cond-d}\\
    &\big(\nabla f(\bs{x}),\bs{x}-\bs{x}^*\big)\geq \zeta \big(f(\bs{x})-f^*\big)\; (\mathrm{QC}),\label{cond-e}\\
    &\big(\nabla f(\bs{x}),\bs{x}-\bs{x}_p\big)\geq \zeta \big(f(\bs{x})-f^*\big)\; (\mathrm{WQC}).\label{cond-f}
\end{eqnarray}
\end{subequations}Here $\mu>0$, $0<\zeta\leq 1$ and $\bs{x}_p$ is a projection of $\bs{x}$ onto $\mathcal{X}$.

Another interesting characterization is in terms of the local geometry 
 \cite{fehrman2020convergence}. Let $\mathcal{M}:=\{\bs{x}:f(\bs{x})=f^*\}$ be the set of all global minima and assume that there exists an open set $U\subset\mathbb{R}^d$ with some $\mathfrak{d}\in\{0,1,..,d-1\}$ such that $\mathcal{M}\cap U$ is a non-empty $\mathfrak{d}$-dimensional $C^2$-submanifold of $\mathbb{R^d}$ with $f \in C^3(U)$ and
\begin{equation}\label{HCPR}
    \textrm{rank}(\nabla^2f(\bs{x}))=d-\mathfrak{d}\textrm{ for each }\bs{x}\in \mathcal{M}\cap U.
\end{equation}In particular, they also considered the case of such a submanifold being compact and without boundary [\cite{fehrman2020convergence}, Section 6]. We will refer to this condition as the Hessian of the constant positive rank on a Compact submanifold (HCPRC) condition.

\medskip
We now discuss the relations between various convexity conditions. 

\begin{rem}
$\mathrm{LRSI}$ and $\mathrm{WQC}$ are special cases of the $\mathrm{NNS}$ condition, since we can choose the following support functions respectively,
\begin{subequations}
\begin{eqnarray}
    &h(\bs{x})=\mu\|\bs{x}-\bs{x}_p\|^2\; (\mathrm{for \; LRSI})\\
    &h(\bs{x})=\zeta (f(\bs{x})-f^*)\; (\mathrm{for \; WQC}).
\end{eqnarray}
\end{subequations}
Therefore, Assumption \ref{as:04} encompasses very general local landscapes for $N_r(\mathcal{X})$ due to the flexibility of choosing $h(\bs{x})$. One can make an interpretation that $h(\bs{x})$ estimates how close the loss value $f(\bs{x})$ is to the minimum $f^*$ by the mean value theorem,
\begin{equation*}
    (\nabla f(\bs{x}),\bs{x}-\bs{x}_p)\approx f(\bs{x})-f^*.
\end{equation*} 
\end{rem}

Now we turn to the relationship between NNS \eqref{gc} and PL$^*$ or QG$^*$ \eqref{conditions}. We first introduce a useful proposition of a projection map.
\begin{prop}\label{projpro}
For a compact set $\mathcal{X}\subset\mathbb{R}^d$ and any $\bs{x}\in\mathbb{R}^d$, let $\bs{x}_p\in\Pi_{\mathcal{X}}(\bs{x})$. If $\bs{x}\neq\bs{x}_p$, then $\Pi_{\mathcal{X}}(\bs{y})=\{\bs{x}_p\}$ for any $\bs{y}\in \{\bs{z}:\bs{z}=(1-t)\bs{x}_p+t\bs{x},\; t\in(0,1)\}$.
That is, the projection of any $\bs{y}$ between $\bs{x}$ and $\bs{x}_p$ onto $\mathcal{X}$ is unique and equal to $\bs{x}_p$.
\end{prop}

In the following example, $f$ represents functions whose local landscape in the vicinity of $\mathcal{X}$ is `highly flat', in which case both  \eqref{cond-e} (WQC) and  \eqref{cond-f} (NNS) conditions hold, but the conditions  \eqref{cond-b} (PL$^*$) and  \eqref{cond-c} (QG$^*$) fail.

\begin{example}[A flat basin of attraction is neither PL$^*$ or QG$^*$ but is WQC]\label{ex1}
Let $f$ be a non-convex, continuously differentiable function and satisfy that $\underset{\bs{x}\in\mathbb{R}^d}{\inf}f(\bs{x})=0$. Let  $\mathcal{X}$ be a compact set. Assume that for a closed $r$-neighborhood, $N_r(\mathcal{X})$, and for some $C>0$ and some $q>2$, 
\begin{equation*}
    f(\bs{x})=C\|\bs{x}-\bs{x}_p\|^q\quad\textrm{ for any }\bs{x}\in N_r(\mathcal{X}),
\end{equation*}
where $\bs{x}_p$ is a projection from $\bs{x}$ to $\mathcal{X}$. Due to the projection and $f\geq 0$, $\mathcal{X}$ consists of global minima of $f(\bs{x})$.  By Proposition \ref{projpro}, it follows that for any $\bs{x}\in N_r(\mathcal{X})$ and any $t\in [0,r]$, 
\begin{equation*}
    f(\bs{x}_p+t\bs{v})=Ct^q, 
\end{equation*}where $\bs{v}$ is the unit direction from $\bs{x}_p$ to $\bs{x}$. Then, by Lemma \ref{hinderlem}, the function $f(\bs{x}_p+t\bs{v})$ in $t$ is $\zeta$-quasar-convex with some $\zeta>0$, since it is continuously differentiable in $[0,r]$ and $f'(t)=0$ only if $t=0$. Finally, note that the form of such a function remains the same as is $Ct^q$ regardless of $\bs{x}_p\in\mathcal{X}$. This implies the NNS condition with $h(\bs{x})=\zeta f(\bs{x})$. As a result, $f(\bs{x})$ is WQC \eqref{conditions} as well as NNS \eqref{gc}. 

On the other hand, a calculus trick yields that
\begin{equation*}
    \|\nabla f(\bs{x})\|^2= |(\nabla f(\bs{x}),\bs{v})|^2=\left(Cqt_0^{q-1}\right)^2,
\end{equation*}where $t_0$ is defined by $\bs{x}=\bs{x}_p+t_0\bs{v}$. The first equality holds since $f$ is defined by the distance $\|\bs{x}-\bs{x}_p\|$ and decreases the fastest towards $\bs{x}_p$ at the point $\bs{x}$, which implies that $v$ and $\nabla f(\bs{x})$ are parallel. Thus, we note that $\|\nabla f(\bs{x})\|^2$ is of order $t_0^{2q-2}$, while $f(\bs{x})$ is of order $t_0^q$. Since $q>2$, we have
\begin{equation}
    \lim_{t\to 0}\frac{\|\nabla f(\bs{x}_p+t\bs{v})\|^2}{f(\bs{x}_p+t\bs{v})}=0.
\end{equation}Thus, PL$^*$ \eqref{conditions} does not hold even near $\mathcal{X}$. Similarly, QG$^*$ \eqref{conditions} does not hold either, since
\begin{equation*}
    \lim_{t\to 0}\frac{f(\bs{x}_p+t\bs{v})}{t^2}=0. 
\end{equation*}
\end{example}

Now we summarize the relationships among the convexity conditions in the following diagram, 
\begin{figure}[htp]
\begin{equation}\label{relationships}
\begin{split}
 (\mathrm{HCPRC}) \overset{(1)}{\Longrightarrow} \;&(\mathrm{LRSI})\overset{(2)}{\Longrightarrow} (\mathrm{NNS})\\
&\quad\Downarrow_{(3)}  \hspace{1.3cm}\Uparrow_{(4)}\\
(\mathrm{QG}^*)\textrm{ or }&\;(\mathrm{PL}^*) \overset{(7)}{\centernot\Longrightarrow} \;(\mathrm{WQC})\\
&\hspace{2.4cm} \Uparrow_{(5)}\\
&\hspace{2.4cm} (\mathrm{QC})\\
&\hspace{2.4cm} \Uparrow_{(6)}\\
&\hspace{2.4cm} (^*\mathrm{C})
\end{split}
\end{equation}
\caption{The inclusion relations among the convexity conditions. 
The proof of the implication of one condition to another: (1) [\cite{fehrman2020convergence},Lemma 14]; (2) Assumption \ref{as:04} with $h(\bs{x})=\mu\|\bs{x}-\bs{x}_p\|^2$; (3) [\cite{karimi2016linear},Appendix A] (Global condtions are given but the same technique also applies; (4)  Assumption \ref{as:04} with $h(\bs{x})=\zeta(f(\bs{x})-f^*)$; (5) WQC \eqref{cond-f} with $\mathcal{X}=\{\bs{x}^*\}$ $\rightarrow$ QC \eqref{cond-e}; (6) [\cite{hinder2020near},Observation 3];  (7) Example \ref{ex1}.  Note that (3) is the implication from LRSI to PL$^*$. The relationship between QG$^*$ and LRSI and that between between LRSI and WQC are open.} 

\end{figure}

\subsection{Main Results on the Convergence of SGD }

The following theorem shows how likely the iterations from SGD will stay close to $\mathcal{X}$ and eventually converge to $\mathcal{X}$ if its local landscape satisfies Assumption \ref{as:04}.

\begin{theorem}[Stochastic Stability and Probabilistic Convergence]\label{SGDthm}Suppose that there exists a closed $r$-neighborhood of a compact subset of global minima $\mathcal{X}$ such that $N_r(\mathcal{X})$ satisfies Assumption \ref{as:04}. Under Assumptions \ref{as:01} to \ref{as:03}, the following statements hold for the SGD iterations \eqref{eq: sgd},
\begin{itemize}
    \item[(i)]For any initial $\bs{x}_1\in N_r(\mathcal{X})$, the $N$ iterates remain in $N_r(\mathcal{X})$ with a positive probability,
\begin{equation}\label{stability}
    \mathbb{P}\left\{\bs{x}_n\in N_r(\mathcal{X})\textrm{ for each }n\in [N]|\bs{x}_1\right\}\geq 1-C_N,
\end{equation}
where 
\begin{equation}\label{eq: CN-bn'}
\begin{split}
    &C_N:=\frac{b_N}{r^2}\left(\mathrm{dist}(\bs{x}_1,\mathcal{X})^2+\frac{\sigma_r}{I}\sum_{n=1}^{N-1}\frac{a_n^2}{b_{n+1}}\right),\\
    &b_n:=\prod_{j=1}^{n-1}(1+L_r^2a_n^2),\quad b_1:=1,
\end{split}
\end{equation}$L_r$ is the Lipschitz constant for $N_r(\mathcal{X})$ and $I$ is the batch-size.
    \item[(ii)]If $h(\bs{x})>0$ for each $\bs{x}\in N_r(\mathcal{X})\setminus\mathcal{X}$ in Assumption \ref{as:04}, then
    \begin{equation}\label{proconv}
    \mathbb{P}\left\{\lim_{n\to\infty}\bs{x}_n=\bs{x},\; f(\bs{x})=f^*\;|\{\bs{x}_n\}_{n=1}^{\infty}\subset N_r(\mathcal{X})\right\}=1.
    \end{equation}
\end{itemize}
Consequently, these statements imply that
\begin{equation*}
    \mathbb{P}\left\{\lim_{n\to\infty}\bs{x}_n=\bs{x},f(\bs{x})=f^*|\bs{x}_1\right\}\geq 1-C_{\infty}.
\end{equation*}
\end{theorem}

\begin{rem}\label{remark1}  
The inequality explicitly shows the role of the hyperparameters (e.g. $\bs{x}_1$, $a_n$, $I$, $\sigma_r$, $L_r$, $r$) \eqref{stability} in the stochastic stability of SGD in non-convex optimization. A qualitative interpretation is that if one starts with a good initialization and the initial guess $\bs{x}_1$ lies in a large nearly convex region (i.e., $r\gg 1$) near a set of global optima $\mathcal{X}$, then the optimization algorithm { has little chance of leaving the region.} Such a nearly convex landscape immediately belongs to the NNS condition \eqref{gc} by the diagram \eqref{relationships} which asserts that NNS includes LRSI. As such, our result \eqref{stability} theoretically strengthens the statement regarding initializations and well-behaved loss regions by \cite{li2018visualizing}. The empirical results there (ResNets and shallow VGG-like nets) also suggest that our theory can be applied to a variety of problems of practical interest.

Furthermore, \cite{li2018over} proved the feature of strictly decreasing loss towards a compact set of global minima in overparameterized fully connected deep neural networks (DNNs) with distinct input samples. This feature can be formulated as the condition $h(\bs{x})>0$ in theorem \ref{SGDthm}. Therefore, our result \eqref{proconv} together their result theoretically explains why SGD can achieve small training error in the regime of overparameterization, especially with the concentration inequality \eqref{concenth} for $h(\bs{x})$ in Theorem \ref{complexsgd}.    
\end{rem}

\begin{rem}
In terms of stochastic stability, we observe that $C_{\infty}$ in the definition \eqref{eq: CN-bn'} can be finite even for some small batch size, e.g. $I=1$, as long as $a_n$ is sufficiently small, since the learning rate $a_n$ satisfies $\sum_na_n^2<\infty$ in Assumption \ref{as:03}. Thus, our result \eqref{stability} theoretically ensures the regime of small constant batch size for SGD \cite{goyal2017accurate} with general learning rate $a_n$, which is in contrast to the similar result by \cite{fehrman2020convergence} that requires a policy of increasing batch-size $I$ as a polynomial order of the number of iterations $N$ to keep the stability of the iteration.
\end{rem}

Theorem \ref{SGDthm} also allows us to examine the stochastic stability of SGD for both decreasing and constant learning rates, as follows.
\begin{cor}[Upper Bounds of Learning Rates for Non-Convex Loss]\label{stabcond}
Under the assumptions in Theorem \ref{SGDthm}, the following statements hold for iterations of SGD
\begin{itemize}
    \item[(i)]If the learning rate is decreasing, $a_n=\frac{a}{n^{\beta}}$ with $\beta\in(\frac{1}{2},1]$ and $n\geq 1$, the stochastic stability that $1-C_{\infty}>0$ in Theorems \ref{SGDthm} is satisfied for any damping parameter $a>0$ such that
\begin{equation}\label{stineqdec}
    e^{\frac{2\beta L_r}{2\beta-1}a^2}\left(\mathrm{dist}(\bs{x}_1,\mathcal{X})^2+\frac{2\beta \sigma_ra^2}{(2\beta-1)I}\right)<r^2.
\end{equation}

    \item[(ii)]If the learning rate is a constant, $a_n=a$ for $n\in[N]$, the stochastic stability that $1-C_N>0$ is satisfied if the learning rate $a$ and the number of iteration $N$ fulfill that
    \begin{equation}\label{stineqcon}
    (1+L_r^2a^2)^{N-1}\left(\mathrm{dist}(\bs{x}_1,\mathcal{X})^2+\frac{\sigma_r}{IL_r^2}\left(1-\left(\frac{1}{1+L_r^2a^2}\right)^{N-1}\right)\right)<r^2.
\end{equation}
    
\end{itemize}

\end{cor}
\begin{rem}In previous works, upper bounds for learning rates in the context of convex optimization have been studied for the optimization error in  \cites{nesterov2003introductory,bottou2018optimization}. As discussed in both works, upper bounds of learning rates are related to the optimal error $f(\bs{x}_n)-f^*$ and the condition number $\frac{\mu}{L}$, where $\mu$ is the parameter in terms of the strong convexity and $L$ is the global Lipschitz constant for gradient.  Corollary \ref{stabcond} suggests another types of upper bound pertaining to the stochastic stability in non-convex optimization. 

In addition, Corollary \ref{stabcond} implies that for large number of iterations ($N\gg 1$), using decreasing learning rates requires less effort than constant learning rates in terms of hyperparameters ( $\bs{x}_1$, $L_r$, $\sigma_r$, $a$, etc.). Specifically, while the parameter $a$ in the second term of \eqref{stineqdec} directly reduces the variance $\sigma_r$, such reduction in \eqref{stineqcon} occurs when a careful balance is struck between the number of iterations $N$ and the batch-size $I$ given a constant learning rate $a$.
\end{rem}

The main idea for proving Theorem \ref{SGDthm} will directly yield a concentration inequality of SGD to achieve $h(\bs{x})\leq\epsilon$ as defined in Assumption \ref{as:04}. We provide more precise concentration properties of SGD for the non-isolated global minima, $\mathcal{X}$, under local conditions as follows.    

\begin{theorem}[Concentration Inequalities of SGD for Non-Convex Loss]\label{complexsgd}
With the same assumptions and the notation in Theorem \ref{SGDthm}, SGD \eqref{eq: sgd} with initialization $\bs{x}_1\in N_r(\mathcal{X})$ satisfies that for any tolerance $\epsilon>0$,
\begin{equation}\label{concenth}
    \mathbb{P}\left\{\min_{1\leq n\leq N}h(\bs{x}_n)>\epsilon|\bs{x}_1\right\}\leq \frac{1}{\epsilon}\cdot\frac{\mathrm{dist}(\bs{x}_1,\mathcal{X})^2+\frac{\sigma_r}{I}\sum_{n=1}^{N}\frac{a_n^2}{b_{n+1}}}{2\sum_{n=1}^N\frac{a_n}{b_{n+1}}}+C_{N+1}. 
\end{equation}

Let $a'>0$ satisfy the inequality \eqref{stineqdec} for the case of decreasing learning rates. 

Then, for the learning rate $a_n=\frac{a}{n^{\beta}}$ with $\beta\in(\frac{1}{2},1)$,
\begin{itemize}
    \item[(i)] if $f$ satisfies the $\mathrm{WQC}$ condition \eqref{cond-f} with $a\in(0,a')$, then
\begin{equation}\label{concentwqc}
    \mathbb{P}\left\{\min_{1\leq n\leq N}f(\bs{x}_n)-f^*>\epsilon|\bs{x}_1\right\}\leq \mathcal{O}\left(\frac{1}{\epsilon N^{1-\beta}}\right)+C_{N+1},
\end{equation}
    \item[(ii)] if $f$ satisfies the $\mathrm{HCPRC}$ condition \eqref{HCPR} with $a\in(0,\min\{\frac{1}{2C},\frac{L_r}{c^2},a'\})$ where $C$ is defined in \eqref{constantC} and $c$ in \eqref{constantc}, then
    \begin{equation}\label{concentHCPR}
    \mathbb{P}\left\{f(\bs{x}_N)-f^*>\epsilon|\bs{x}_1\right\}\leq\mathcal{O}\left(\frac{1}{\epsilon N^{\beta}}\right)+C_{N+1}. 
\end{equation}More generally, the same concentration result \eqref{concentHCPR} is obtained if $f$ satisfies the $\mathrm{LRSI}$ condition \eqref{cond-a}.
\end{itemize}

\end{theorem}

\begin{rem}
The first result of Theorem \ref{complexsgd} can quantitatively measures the slowdown of the convergence of SGD within a highly flat basin of attraction, which depends on the loss function $f$, the initialization $\bs{x}_1$, the learning rate $a_n$, etc. As discussed after Assumption \ref{as:04}, the function $h$ can be an implicit estimate for the optimization error. For example, in the case of a large flat landscape around global minima, $h$ can be approximately equal to $\mathcal{O}\left(\|\bs{x}-\bs{x}_p\|^{q(\bs{x})}\right)$ for some degree function $q(\bs{x})\gg 1$ depending on $\bs{x}$. In this case, to reach a tolerance $\epsilon>0$, SGD needs a very large number of iterations, since the right hand side of \eqref{concenth} is of order $\mathcal{O}\left(\frac{1}{\epsilon^qN^{1-\beta}}\right)$ for some large degree $q\gg 1$. For a general WQC \eqref{cond-f} loss function, we can make the same interpretation with \eqref{concentwqc}.

On the other hand, if the loss function $f$ starts with good initialization near which a submanifold of global minima fulfills the HCPRC condition \eqref{HCPR}, then SGD shows better convergence rate as shown in \eqref{concentHCPR} with order of $\mathcal{O}\left(\frac{1}{\epsilon N^{\beta}}\right)$. As $\beta$ approaches $1$, the rate of convergence becomes sublinear,  $\mathcal{O}(\frac{1}{\epsilon N})$ with the first term (The second term accounts for the stochastic stability \eqref{stability}). If we focus on the first term in \eqref{concentHCPR}, then the result \eqref{concentHCPR} generalizes the convergence result for strongly convex functions with isolated minima \cites{JainNN19,nemirovski2009robust}. More importantly, our result \eqref{concentHCPR} improves the bound in \cite{fehrman2020convergence} whose concentration inequality involves the term $\mathcal{O}(N^{1-\beta})$ under equivalent assumptions. 
\end{rem}

\section{Convergence Analysis}
In this section, we give the proofs of the main result. First, we briefly review the notation for stopped stochastic processes introduced in [\cite{kushner2003stochastic}, Section 4.5]. 

\subsection{Stopped Stochastic Processes}\label{notation4}

Let $\bs{x}_1$ be the initialization and $\{(\bs{x}_n,\mathcal{F}_n)\}_{n=1}^{\infty}$ be the iteration from SGD or a stochastic process with a filtration. Let $V(\cdot),k(\cdot)$ be real-valued and non-negative functions on $\mathbb{R}^d$. Especially, $V(\cdot)$ will represent a Lyapunov function. In addition, we will denote a perturbed Lyapunov function by $V_n(\bs{x}_n)$ and a non-negative function scaled with learning rates by $k_n(\bs{x}_n)$. More importantly, any function or stochastic process with the tilde superscripted will stand for a modified function in conjunction with a stopped process, respectively, which depend on a stopping time, specifically  
\begin{equation}\label{stopeq}
    \tilde{\bs{x}}_n:=\begin{cases}\bs{x}_n,\; n\leq \tau\\\bs{x}_{\tau},\; n>\tau ,\end{cases}\tilde{V}(\tilde{\bs{x}}_n):=\begin{cases}V_n(\bs{x}_n),\;n\leq\tau\\V_{\tau}(\bs{x}_{\tau}),\;n>\tau,\end{cases}\tilde{k}(\bs{x}):=\begin{cases}k(\bs{x}),\;\bs{x}\in K\\ 0,\;\textrm{ otherwise}.\end{cases}
\end{equation}with the stopping time $\tau:=\{n\geq 1: \bs{x}_n\not\in K\}$ for a given set $K$. Throughout the convergence analysis, we will consider $K$ as some set of global minima and $\mathbb{I}_A$ as the indicator function which values one on the event $A$ and zero otherwise.

\subsection{Outline of the Proof}
For the proof of Theorem \ref{SGDthm}, we first construct a recursive inequality in terms of the distance between $\bs{x}_n$ and $\mathcal{X}$, which essentially yields a supermartingale property as in Lemma \ref{kushnerlem},
\begin{equation}
    \mathbb{E}[\textrm{dist}(\bs{x}_{n+1},\mathcal{X})^2|\bs{x}_n\in N_r(\mathcal{X})]\lesssim\mathbb{E}[\textrm{dist}(\bs{x}_{n},\mathcal{X})^2]-a_nh(\bs{x}_n),
\end{equation}where $h$ is a non-negative function as in Assumption \ref{as:04} and $a_n$ is the learning rate. This property brings the problem into the framework of stochastic stability \cite{kushner2003stochastic}. In particular, a Lyapunov function $V$ and a non-negative function $k$ can be defined from $\textrm{dist}(\bs{x}_{n},\mathcal{X})^2$ and $h(\bs{x}_n)$, respectively.  Our next step is to use Lemma \ref{kushnerlem} to estimate the probability of divergence from $N_r(\mathcal{X})$, i.e., the probability,
\begin{equation}
    \mathbb{P}\left\{\textrm{dist}(\bs{x}_{n},\mathcal{X})>r|\bs{x}_1\right\}.
\end{equation}This gives the first result in Theorem \ref{SGDthm}. 

For the second result, we observe that  any stable path $\{\bs{x}_n\}$ converges by Lemma \ref{Patellem}.  In such an event, $k(\bs{x}_n)$ converges and it remains to show that $k(\bs{x}_n)\approx h(\bs{x}_n)$ will converge to $0$ if $h(\bs{x})>0$ is assumed in the region $N_r(\mathcal{X})\setminus\mathcal{X}$. We prove it by contradiction that $k(\bs{x}_n)$ converges to a positive random variable with positive probability. On the other hand, telescoping the above inequality yields an inequality of the form
\begin{equation}
    \infty=\left(\sum_{n=N}^{\infty}a_n\right)\delta\mathbb{P}\left\{k(\bs{x}_n)>\delta\textrm{ for all }n\geq N,\{\bs{x}_n\}_{n=2}^{\infty}\subset B_r(\bs{x}^*)|\bs{x}_1\right\}\lesssim \textrm{dist}(\bs{x}_{1},\mathcal{X})^2<\infty.
\end{equation}The left hand side is infinite by Lemma \ref{kolem} and Assumption \ref{as:03}, while the right hand side is certainly finite. 

\medskip

For Theorem \ref{complexsgd}, we use the telescoping trick above and obtain an inequality in expectation as follows,
\begin{equation}
    \mathbb{E}\left[\min_{1\leq n\leq N}h(\bs{x}_n)\mathbb{I}_{E_N}|\bs{x}_1\right]\lesssim \frac{\textrm{dist}(\bs{x}_{1},\mathcal{X})^2}{\sum_{n=1}^Na_n}.
\end{equation} Note that $E_N$ represents the event in which the iterations remain near $\mathcal{X}$ up to the $N$-th step. By applying the Markov's inequality, we arrive at the concentration inequality, 
\begin{equation*}
    \mathbb{P}\left\{\min_{1\leq n\leq N}h(\bs{x}_n)\leq\epsilon\textrm{ or }E_N\textrm{ does not occur}|\bs{x}_1\right\}\gtrsim 1-\frac{\textrm{dist}(\bs{x}_{1},\mathcal{X})^2}{\epsilon\sum_{n=1}^Na_n}.
\end{equation*}However, by Theorem \ref{SGDthm}, we know an upper bound for the probability of the event that the iterations diverge. Combining these results, we find the estimate,
\begin{equation*}
    \mathbb{P}\left\{\min_{1\leq n\leq N}h(\bs{x}_n)\leq\epsilon|\bs{x}_1\right\}\lesssim \frac{\textrm{dist}(\bs{x}_{1},\mathcal{X})^2}{\epsilon\sum_{n=1}^Na_n}+C_{N+1}.
\end{equation*}

Thanks to the fact that the NNS condition \eqref{cond-f} is weaker than the WQC condition \eqref{cond-e}, we set $h(\bs{x})=\zeta (f(\bs{x})-f^*)$ and obtain the inequality by using the integral test for $a_n=\frac{a}{n^{\beta}}$. For the case of the HCPR condition \eqref{HCPR}, it suffices to show the same result under the LRSI condition \eqref{conditions}, since it is weaker by Lemma \ref{fehrlem}. Observing that PL$^*$ is weaker than LRSI in \eqref{relationships}, Lemma \ref{Patellem} gives the recursive inequality of a form 
\begin{equation}
    \mathbb{E}[V_{n+1}|\mathcal{F}_n]\lesssim (1-a_n)V_n+a_n^2
\end{equation}and by using Lemma \ref{chunglem}, we obtain the result under the HCPRC condition.

\subsection{Some useful Lemmas and Theorems}
In this section, we provide some prior lemmas and theorems. Some of them are restated to suitably use for the proofs of main results in the next section.
\begin{lemma}[\cite{kushner2003stochastic}, Theorem 5.1]\label{kushnerlem}
Let $\{\bs{x}_n\}$ be a Markov chain on $\mathbb{R}^d$. Let $V(\cdot)$ be a non-negative real-valued function on $\mathbb{R}^d$ and for a given $r>0$, define the set $N_r:=\{\bs{x}:V(\bs{x})\leq r\}$. Suppose that for any $\bs{x}\in N_r$ and each $n\geq 1$,
\begin{equation}
    \mathbb{E}[V(\bs{x}_{n+1})|\bs{x}_n]-V(\bs{x})\leq -k(\bs{x}),
\end{equation}where $k(\bs{x})\geq 0$ and is continuous on $N_r$. Then, for any $\bs{x}_n\in N_r$,
\begin{equation}
    \mathbb{P}\left\{\sup_{\nu+1\leq n<\infty}V(\bs{x}_n)|\bs{x}_{\nu}\right\}\leq\frac{V(\bs{x}_{\nu})}{r}.
\end{equation}
\end{lemma}

\begin{lemma}[\cite{ko2021stochastic}, Lemma E.1]\label{kolem}Let $\{k_n\}$ be a non-negative sequence of random variables. If $\mathbb{P}\left\{\liminf_{n}k_n>\delta\right\}>0$ for some $\delta>0$, then there exists a natural number $N\geq 1$ such that
\begin{equation}
    \mathbb{P}\left\{k_n>\delta\textrm{ for all }n\geq N|\liminf_{n}k_n>\delta\right\}>0.
\end{equation}
\end{lemma}

\begin{lemma}[\cite{patel2021global}, Theorem 1]\label{capturethm}Under Assumptions \ref{as:01} to \ref{as:03}, for any $\bs{x}_1\in K$, the iteration from SGD \eqref{eq: sgd} satisfies the property that
\begin{equation*}
\mathbb{P}\left\{\lim_{n\to\infty}\bs{x}_n=\bs{x},\; \nabla f(\bs{x})=0\;|\;\{\bs{x}_n\}_{n=2}^{\infty}\subset K,\bs{x}_1\right\}=1.
\end{equation*}That is, the iteration almost surely converges to a critical point in the event that it stays in $K$.
\end{lemma}

\begin{lemma}[\cite{patel2021global}, Lemma 2 when $\nabla f$ is locally Lipschitz]\label{Patellem}Denote by $E_n:=\left\{\bs{x}_i\in N_r(\mathcal{X}),\;\forall\; i\in [n]\right\}$, the event that iterations are stable up to the $n$-th step.
Under Assumptions \ref{as:01} to \ref{as:03}, for $\delta>0$, SGD \eqref{eq: sgd} has the property that for any $n\geq 1$,
\begin{equation}\label{SGDpro1}
\begin{split}
    &\mathbb{E}\left[(f(\bs{x}_{n+1})-f^*)\mathbb{I}_{E_{n+1}}|\mathcal{F}_n\right]\\
    &\leq \left(f(\bs{x}_{n})-f^*-a_n(1-Ca_n)\|\nabla f(\bs{x}_n)\|^2+\sigma_{r+\delta}Ca_n^2\right)\mathbb{I}_{E_n}.
\end{split}
\end{equation} Specifically, the constant is given by, 
\begin{equation}\label{constantC}
    C=\frac{L_{r+\delta}}{2}+\frac{\sup_{\bs{x}\in N_r(\mathcal{X})}\|\nabla f(x)\|}{\delta}.
\end{equation}$\sigma_{r+\delta}$ is the constant associated with $N_{r+\delta}(\mathcal{X})$ by Assumption \ref{as:02}.
\end{lemma}

\begin{lemma}[\cite{chung1954stochastic}, Lemma 4]\label{chunglem}
Let $\{e_n\}$ be a non-negative sequence such that 
\begin{equation}
    e_{n+1}\leq \left(1-\frac{C}{n^{\beta}}\right)e_n+\frac{C'}{n^{\beta'}},
\end{equation}where $\beta\in(0,1)$, $\beta<\beta'$ and $C,C'>0$. Then, there exist constants $C''>0$ and $n_0\in\mathbb{N}$ such that \begin{equation}
        e_n\leq \frac{C''}{n^{\beta'-\beta}},\;\textrm{ for each }n\geq n_0.
    \end{equation}
\end{lemma}

\begin{lemma}[\cite{fehrman2020convergence}, Lemma 14 in the case of a compact submanifold]\label{fehrlem}
If $f$ satisfies the HCPRC condition \eqref{HCPR}, for any $\bs{x}_0\in\mathcal{M}\cap U$, then the compact submanifold $\mathcal{X}:=\mathcal{M}\cap U$ has the following property that for some $r>0$ and $c>0$,
\begin{equation}
    \left(\nabla f(\bs{x}),\bs{x}-\bs{x}_p \right)\geq c\|\bs{x}-\bs{x}_p\|^2,
\end{equation}for any $\bs{x}\in N_r(\mathcal{X})$.

\end{lemma}

\subsection{Proofs of Main Results}
\label{prSGD}

\begin{proof}[Proof of Theorem \ref{SGDthm}]
First, we derive a recursive inequality for the distance between $\bs{x}_n$ from SGD \eqref{eq: sgd} and $\bs{x}^*$ as follows,
\begin{equation}\label{quasarineq}
\begin{split}
&\mathbb{E}\left[\textrm{dist}(\bs{x}_{n+1},\mathcal{X})^2|
    \bs{x}_n\right]\leq \mathbb{E}\left[\|\bs{x}_{n+1}-(\bs{x}_n)_p\|^2|\bs{x}_n\right]\\
    &\leq \textrm{dist}(\bs{x}_n,\mathcal{X})^2-2 a_n(\nabla f(\bs{x}_n),\bs{x}_n-(\bs{x}_n)_p)+a_n^2\|\nabla f(\bs{x}_n)\|^2+\sigma_{r}a_n^2\\
    &\leq (1+L_r^2a_n^2)\textrm{dist}(\bs{x}_n,\mathcal{X})^2-2 a_n(\nabla f(\bs{x}_n),\bs{x}_n-(\bs{x}_n)_p)+\sigma_ra_n^2
\end{split}
\end{equation}for any $\bs{x}_n\in N_r(\mathcal{X})$. Here, $L_r$ is the local Lipschitz constant from Assumption \ref{as:01} and $\sigma_r$ can be chosen by Assumption \ref{as:02}. The first inequality holds by the definition of the distance to a set from a point. The second inequality is a direct calculation using \eqref{eq: sgd} and Assumption \ref{as:02}. The conditional expectation of the cross term with $\nabla f(\bs{x}_n)$ and $\xi_n$ vanishes by Assumption \ref{as:02}.  In the last inequality, we used the Lipschitz condition in Assumption \ref{as:01}.

To proceed, let us define the following, 
\begin{subequations}\label{defnoffuncs}
\begin{eqnarray}
    &b_n:=\prod_{j=1}^{n-1}(1+L_r^2a_n^2),\quad b_1:=1,\\
    &V_n(\bs{x}_n):=\frac{\textrm{dist}(\bs{x}_n,\mathcal{X})^2}{b_n}+\sigma_K\sum_{j=n}^{N-1}\frac{a_j^2}{b_{j+1}},\\
    &k_n(\bs{x}_n):=\frac{a_n}{b_{n+1}}k(\bs{x}_n),\quad k(\bs{x}_n):=2(\nabla f(\bs{x}),\bs{x}-\bs{x}_p).
\end{eqnarray}
\end{subequations}

Based on the stochastic stability analysis [\cite{kushner2003stochastic}, Theorem 5.1], by setting the stopping time $\tau:=\{n\geq 2: \bs{x}_n\not\in N_r(\mathcal{X})\}$ with $K=N_r(\mathcal{X})$ in the inequality \eqref{stopeq}, we can modify the inequality \eqref{quasarineq} as follows
\begin{equation}\label{stabineq}
    \mathbb{E}[\tilde{V}_{n+1}(\tilde{\bs{x}}_{n+1})|\mathcal{F}_n]\leq \tilde{V}_n(\tilde{\bs{x}}_n)-\tilde{k}_n(\tilde{\bs{x}}_n)\leq \tilde{V}_n(\tilde{\bs{x}}_n),
\end{equation}which implies that $\{\tilde{V}_n(\tilde{\bs{x}}_n)\}$ is a non-negative supermartingale. Furthermore, by applying the Markov's inequality to this supermartingale, we find that for any $\bs{x}_1\in N_r(\mathcal{X})$,
\begin{equation*}
    \mathbb{P}\left\{\sup_{N\geq n\geq 2}\tilde{V}_n(\tilde{\bs{x}}_n)>\frac{r^2}{b_{N}}|\bs{x}_1\right\}\leq \frac{b_{N}}{r^2}\left(V_1(\bs{x}_1)\right).
\end{equation*}On the other hand, we observe that
\begin{align*}
\mathbb{P}\left\{\sup_{N\geq n\geq 2}\textrm{dist}(\bs{x}_n,\mathcal{X})>r|\bs{x}_1\right\}\leq& \mathbb{P}\left\{\sup_{N\geq n\geq 2}\textrm{dist}(\tilde{\bs{x}}_n,\mathcal{X})>r|\bs{x}_1\right\} \\
    \leq& \mathbb{P}\left\{\sup_{N\geq n\geq 2}\tilde{V}_n(\tilde{\bs{x}}_n)>\frac{r^2}{b_{N}}|\bs{x}_1\right\}.
\end{align*}The first inequality follows by the definition of the stopping time $\tau$ in the above. The second inequality can be deduced from the definition \eqref{defnoffuncs}. Therefore, the first statement follows by combining these two inequalities. Moreover, Assumption \ref{as:03} guarantees the limiting case $N=\infty$. 

For the second statement, we telescope the inequality \eqref{stabineq} and recall  the definition $k_n(\bs{x}_n)$ in  \eqref{defnoffuncs}:
\begin{equation}\label{kineq}
    \sum_{n=1}^{\infty} \frac{a_n}{b_{n+1}}\mathbb{E}[\tilde{k}(\tilde{\bs{x}}_n)|\bs{x}_1]\leq V_1(\bs{x}_1)<\infty.
\end{equation}Note that conditioned on $\bs{x}_1\in N_r(\mathcal{X})$, the event that $\{\bs{x}_n\}_{n=2}^{\infty}\subset N_r(\mathcal{X})$ occurs with some positive probability thanks to the first statement. In this event, $\{\bs{x}_n\}$ converges in $N_r(\mathcal{X})$ with probability $1$ by Lemma \ref{capturethm}. This guarantees $f(\bs{x}_n)$ converges to some random variable by the continuity of $f$, which is less than or equal to the global minimum $f^*$. In fact, we show that $f(\bs{x}_n)$ converges to $f^*$ almost surely in this event. 

Suppose on the contrary that $k(\bs{x}_n)$ in the definition \eqref{defnoffuncs} converges to a positive random variable with some positive probability.  By the continuity of a measure, there exists a $\delta>0$ such that
\begin{equation*}
    \mathbb{P}\left\{\lim_{n\to\infty}k(\bs{x}_n)>\delta|\{\bs{x}_n\}_{n=2}^{\infty}\subset N_r(\mathcal{X}),\bs{x}_1\right\}>0.
\end{equation*}By Lemma \ref{kolem}, there exists a  $N\in\mathbb{N}$ satisfying
\begin{equation*}
    \mathbb{P}\left\{k(\bs{x}_n)>\frac{\delta}{2}\textrm{ for all }n\geq N|\lim_{n\to\infty}k(\bs{x}_n)>\delta,\{\bs{x}_n\}_{n=2}^{\infty}\subset N_r(\mathcal{X}),\bs{x}_1\right\}>0.
\end{equation*}
These two inequalities imply that with some positive probability the event $k(\bs{x}_n)>\frac{\delta}{2}$ for all $n\geq N$ occurs when $\{\bs{x}_n\}_{n=2}^{\infty}\subset B_r(\bs{x}^*)$ given $\bs{x}_1\in B_r(\bs{x}^*)$. However, by the definition of the conditional probability, we have
\begin{equation*}
\begin{split}
    &\mathbb{P}\left\{k(\bs{x}_n)>\frac{\delta}{2}\textrm{ for all }n\geq N,\{\bs{x}_n\}_{n=2}^{\infty}\subset B_r(\bs{x}^*)|\bs{x}_1\right\}\\
    &=\mathbb{P}\left\{k(\bs{x}_n)>\frac{\delta}{2}\textrm{ for all }n\geq N|\{\bs{x}_n\}_{n=2}^{\infty}\subset B_r(\bs{x}^*),\bs{x}_1\right\}\mathbb{P}\left\{\{\bs{x}_n\}_{n=2}^{\infty}\subset B_r(\bs{x}^*)|\bs{x}_1\right\}\\
    &>0, \textrm{ by the above inequality and the previous result of stability}.
\end{split}
\end{equation*}
Furthermore, by using the inequality \eqref{kineq}, Assumption \ref{as:02} and the Markov's inequality, one has,
\begin{equation*}
\begin{split}
    &\infty=\left(\sum_{n=N}^{\infty} \frac{a_n}{b_{n+1}}\right)\frac{\delta}{2}\mathbb{P}\left\{k(\bs{x}_n)>\frac{\delta}{2}\textrm{ for all }n\geq N,\{\bs{x}_n\}_{n=2}^{\infty}\subset B_r(\bs{x}^*)|\bs{x}_1\right\}\\
    &\leq V_1(\bs{x}_1)<\infty,
\end{split}
\end{equation*}which is a contradiction. Thus, in the event that $\{\bs{x}_n\}_{n=2}^{\infty}\subset N_r(\mathcal{X})$, $\lim_{n\to\infty}k(\bs{x}_n)=k(\bs{x})=0$ with probability $1$. Finally, if $\bs{x}\in N_r(\mathcal{X})-\mathcal{X}$, then the assumption that $h(\bs{x})>0$ for all $\bs{x}\in N(\mathcal{X})-\mathcal{X}$ in the statement implies that $0=\frac{k(\bs{x})}{2}=(\nabla f(\bs{x}),\bs{x}-\bs{x}_p)\geq h(\bs{x})>0$, which is a contradiction. Therefore, the limit $\bs{x}$ must lie in $\mathcal{X}$ and $f(\bs{x})=f^*$.
\end{proof}

\begin{proof}[Proof of Corollary \ref{stabcond}]
By the integral test, we have
\begin{equation*}
    \sum_{n=1}^{\infty}\frac{1}{n^{2\beta}}\leq 1+\int_1^{\infty}\frac{dx}{x^{\beta}}=\frac{2\beta}{2\beta-1}.
\end{equation*}
With this, we use the inequality \eqref{stineqdec} and trace back to the condition of stability
\begin{equation*}
\begin{split}
    &r^2>e^{\frac{2\beta L_r}{2\beta-1}a^2}\left(\textrm{dist}(\bs{x}_1,\mathcal{X})^2+\frac{2\beta \sigma_r}{2\beta-1}a^2\right)\geq e^{L_r\sum_{n=1}^{\infty}a_n^2}\left(\textrm{dist}(\bs{x}_1,\mathcal{X})^2+\sigma_r\sum_{n=1}^{\infty}a_n^2\right)\\
    &\geq b_{\infty}\left(\textrm{dist}(\bs{x}_1,\mathcal{X})^2+\sigma_r\sum_{n=1}^{\infty}\frac{a_n^2}{b_{n+1}}\right)\Longrightarrow 1-C_{\infty}>0,
\end{split}
\end{equation*}where $C_{\infty}$ is defined in Theoerem \ref{SGDthm}. In the third inequality, we used the well-known inequality that $\prod_n (1+c_n)\leq e^{\sum_nc_n}$ for any real-valued sequence $\{c_n\}$ with $c_n>-1$. 

In the case of a constant learning rate, we use the well-known formula
\begin{equation}
    1+x+\cdots+x^n=\frac{1-x^n}{1-x}
\end{equation}and a direct calculation shows the result \eqref{stineqcon}.
\end{proof}

\begin{proof}[Proof of Theorem \ref{complexsgd}]
From the inequality \eqref{quasarineq} and its counterpart \eqref{stabineq} with the stopped process, a telescoping trick can be used:
\begin{equation}\label{firstresult}
    \sum_{n=1}^N\mathbb{E}[k_n(\bs{x}_n)\mathbb{I}_{E_N}|\bs{x}_1]\leq\sum_{n=1}^N\mathbb{E}[\tilde{k}_n(\tilde{\bs{x}}_n)|\bs{x}_1]\leq V_1(\bs{x}_1).
\end{equation}By noting that $k_n(\bs{x}_n)=\frac{2a_n}{b_{n+1}}h(\bs{x}_n)$ in \eqref{defnoffuncs}, we have from above inequality
\begin{equation}
    \mathbb{E}[\min_{1\leq n\leq N}h(\bs{x}_n)\mathbb{I}_{E_N}|\bs{x}_1]\leq \frac{V_1(\bs{x}_1)}{2\sum_{n=1}^N\frac{a_n}{b_{n+1}}},
\end{equation}where $V_1(\bs{x}_1)$ is defined in \eqref{defnoffuncs}. This proves the first result.

For the rest of results in the theorem, we use the stabililty result in Theorem \ref{SGDthm} as well as the Markov's inequality. For simplicity, we keep using the above notations. By applying the Markov's inequality to the inequality \eqref{firstresult}, we have
\begin{equation*}
    \mathbb{P}\left\{\min_{1\leq n\leq N}h(\bs{x}_n)>\epsilon\textrm{ and }E_N\textrm{ occurs}|\bs{x}_1\right\}\leq \frac{V_1(\bs{x}_1)}{2\epsilon\sum_{n=1}^N\frac{a_n}{b_{n+1}}},
\end{equation*}or equivalently,
\begin{equation*}
    \mathbb{P}\left\{\min_{1\leq n\leq N}h(\bs{x}_n)\leq\epsilon\textrm{ or }E_N\textrm{ does not occur}|\bs{x}_1\right\}\geq 1-\frac{V_1(\bs{x}_1)}{2\epsilon\sum_{n=1}^N\frac{a_n}{b_{n+1}}}.
\end{equation*}However, according to the result in Theorem \ref{SGDthm}, we see that
\begin{equation*}
    \mathbb{P}\left\{E_N\textrm{ does not occur}|\bs{x}_1\right\}\leq C_{N+1},
\end{equation*}which leads to 
\begin{equation*}
    \mathbb{P}\left\{\min_{1\leq n\leq N}h(\bs{x}_n)\leq\epsilon|\bs{x}_1\right\}\geq 1-\frac{V_1(\bs{x}_1)}{2\epsilon\sum_{n=1}^N\frac{a_n}{b_{n+1}}}-C_{N+1}.
\end{equation*}Thus, if $f$ is WQC and $a_n=\frac{a}{n^{\beta}}$, then we have
\begin{equation}
    \mathbb{P}\left\{\min_{1\leq n\leq N}f(\bs{x}_n)-f^*>\epsilon|\bs{x}_1\right\}\leq \frac{V_1(\bs{x}_1)}{2\zeta\epsilon\sum_{n=1}^N\frac{a_n}{b_{n+1}}}+C_{N+1}=\mathcal{O}\left(\frac{1}{\epsilon N^{1-\beta}}\right)+C_{N+1}.
\end{equation}by the integral test.

Secondly, we suppose that $f$ is HCPRC \eqref{HCPR}. By Lemma \ref{fehrlem}, $f$ is LRSI, that is, there exists a compact set of global minima $\mathcal{X}$, $r>0$ and $c>0$ such that $\bs{x}_0\in\mathcal{X}$ and
\begin{equation}\label{constantc}
    \left(\nabla f(\bs{x}),\bs{x}-\bs{x}_p \right)\geq c\|\bs{x}-\bs{x}_p\|^2,
\end{equation}for any $\bs{x}\in N_r(\mathcal{X})$. This implies that
\begin{equation*}
    \|\nabla f(\bs{x})\|\geq c\|\bs{x}-\bs{x}_p\|,
\end{equation*}by the Cauchy-Schwarz inequality. By Lipschitz continuity in Assumption \ref{as:01}, we have
\begin{equation*}
    f(\bs{x})\leq f^*+\frac{L_r}{2}\|\bs{x}-\bs{x}_p\|^2,
\end{equation*}since $\nabla f(\bs{x}_p)=0$. With these results, we achieve that for any $\bs{x}\in N_r(\mathcal{X})$,
\begin{equation*}
    f(\bs{x})-f^*\leq \frac{L_r}{2}\|\bs{x}-\bs{x}_p\|^2\leq \frac{L_r}{2c^2}\|\nabla f(\bs{x})\|^2,
\end{equation*}which is the PL$^*$ condition \eqref{cond-b}. Then, for learning rate $a_n\in (0,\min\{\frac{1}{2C},\frac{L_r}{c^2}\})$ ($C$ defined in \eqref{constantC} and $c$ in \eqref{constantc}), by Lemma \ref{Patellem}, we have
\begin{equation}
    \mathbb{E}[(f(\bs{x}_{n+1})-f^*)\mathbb{I}_{E_{n+1}}|\mathcal{F}_n]\leq (1-\frac{c^2a_n}{L_r})(f(\bs{x}_n)-f^*)\mathbb{I}_{E_{n}}+\frac{\sigma_{r+\delta}C}{I}a_n^2,
\end{equation}and
\begin{equation}
    \mathbb{E}[V_{n+1}|\bs{x}_1]\leq (1-\frac{c^2a_n}{L_r})\mathbb{E}[V_n|\bs{x}_1]+\frac{\sigma_{r+\delta}C}{I}a_n^2,
\end{equation}by letting $V_n:=(f(\bs{x}_n)-f^*)\mathbb{I}_{E_n}$ and taking expectation up to $\bs{x}_1$.

In particular, by Lemma \ref{chunglem}, if $a_n=\frac{a}{n^{\beta}}$ with $\beta\in(0,1)$ and $a\in(0,\min\{\frac{1}{2C},\frac{L_r}{c^2}\})$, there exist $C''>0$ and $N_0\in\mathbb{N}$ such that
\begin{equation*}
    \mathbb{E}[V_N|\bs{x}_1]\leq \frac{C''}{N^{\beta}}, \textrm{ for any } N\geq N_0
\end{equation*}and by using the Markov's inequality,
\begin{equation*}
    \mathbb{P}\left\{f(\bs{x}_{N})-f^*>\epsilon\textrm{ and }E_N\textrm{ occurs}|\bs{x}_1\right\}\leq \frac{C''}{\epsilon N^{\beta}}.
\end{equation*}As we did above, we achieve that for any  $N\geq N_0$,
\begin{equation*}
    \mathbb{P}\left\{f(\bs{x}_{N})-f^*\leq\epsilon|\bs{x}_1\right\}\geq 1-\frac{C''}{\epsilon N^{\beta}}-C_{N+1}.
\end{equation*}
Finally, suppose that $f$ satisfies the LRSI condition \eqref{cond-a}. Then, based on the proof from the inequality \eqref{constantc} to the above concentration inequality, we can obtain the same concentration result for the case of the LRSI condition \eqref{cond-a}.   
 
\end{proof}

\section{Summary and Discussions }

This paper has focused on local convergence in the context of non-convex optimization, with the stochastic gradient descent (SGD), especially for loss functions with non-isolated minima,  We have proved the convergence and  concentration inequalities in terms of hyperparameters. The technical results rely on the stochastic stability analysis and the optional stopping theorem for discrete stochastic processes and geometric characterizations of non-isolated minima. 
An extension of our analysis to other variants of stochastic gradient method is likely, and it is expected to provide sufficient conditions on the learning rates and the local Lipschitz constant to guarantee the convergence with high probability.

\section*{Acknowledgments}
We would like to acknowledge support for this project
from the National Science Foundation (NSF grant DMS-1953120).

\appendix

\section{Appendix}\label{Alemmas}

\begin{proof}[Proof of Proposition \ref{projpro}]
To start, we show that $\Pi_{\mathcal{X}}(\bs{y})$ contains $\bs{x}_p$. Note that for any $\bs{z}\in \mathcal{X}$, $\|\bs{x}-\bs{x}_p\|\leq \|\bs{x}-\bs{z}\|$ by the definition of the projection. From this, the colinear relation between $\bs{x}$, $\bs{y}$ and $\bs{x}_p$ implies that
\begin{equation*}
    \|\bs{y}-\bs{x}_p\|+\|\bs{x}-\bs{y}\|=\|\bs{x}-\bs{x}_p\|\leq\|\bs{x}-\bs{z}\|. 
\end{equation*} Furthermore, by the triangle inequality,
\begin{equation*}
    \|\bs{y}-\bs{x}_p\|\leq\|\bs{x}-\bs{z}\|-\|\bs{x}-\bs{y}\|\leq\|\bs{y}-\bs{z}\|. 
\end{equation*}That is, $\Pi_{\mathcal{X}}(\bs{y})$  at least contains $\bs{x}_p$. 

Now, we prove that this set is indeed a singleton. Suppose on the contrary that there is $\bs{y}_p\in \Pi_{\mathcal{X}}(\bs{y})$ and $\bs{y}_p\neq\bs{x}_p$. We show that $\bs{y}_p$ does not lie on the line through $\bs{x}_p$ and $\bs{x}$. First of all, $\bs{y}_p$ cannot lie on $(\bs{x}_p,\bs{x})$. Otherwise, $\bs{x}_p$ is not projection of $\bs{x}$. The other possibility is that $\bs{y}_p=\bs{x}_p+t(\bs{x}-\bs{x}_p)$ for some $t>1$. However, we can see that
\begin{equation*}
    \|\bs{y}-\bs{x}_p\|<\|\bs{x}-\bs{x}_p\|\leq \|\bs{x}-\bs{y}_p\|<\|\bs{x}-\bs{y}_p\|+\|\bs{x}-\bs{y}\|=\|\bs{y}-\bs{y}_p\|. 
\end{equation*}In the second inequality, we recall the fact that $\bs{x}_p$ is a projection of $\bs{x}$ onto $\mathcal{X}$. This inequality is not true as opposes to the hypothesis that both sides must be equal as dist$(\bs{y},\mathcal{X})$. 

As a result, we can assume that $\bs{y}_p$ is not on the line passing through $\bs{x}$, $\bs{y}$ and $\bs{x}_p$. However, this results in the strict triangle inequality,
\begin{equation*}
    \|\bs{x}-\bs{y}_p\|<\|\bs{x}-\bs{y}\|+\|\bs{y}-\bs{y}_p\|.
\end{equation*}This inequality, together with  $\|\bs{y}-\bs{y}_p\|=\|\bs{y}-\bs{x}_p\|=\textrm{dist}(\bs{y},\mathcal{X})$, leads to a contradiction to  $\bs{x}_p\in \Pi_{\mathcal{X}}(\bs{x})$, i.e.,
\begin{equation*}
    \|\bs{x}-\bs{y}_p\|<\|\bs{x}-\bs{x}_p\|.
\end{equation*}This completes the proof.
\end{proof}

\begin{lemma}[\cite{hinder2020near}, Observation 1]\label{hinderlem}
Let $a<b$ and let $f$ be a real-valued continuously differentiable function on $[a,b]$. Then, $f$ is unimodal, that is, $f'(c)\neq 0$ for all $c\in[a,b]$ such that $c\not\in \textrm{argmin}_{x\in[a,b]}f(x)$ if and only if $f$ satisfies the QC for some $\zeta\in(0,1]$ and some minimum $x^*\in[a,b]$,
\begin{equation*}
    f'(x)(x-x^*)\geq \zeta(f(x)-f^*).
\end{equation*}
\end{lemma}

\bibliographystyle{unsrt}  
\bibliography{references}

\end{document}